\documentclass{article}

\PassOptionsToPackage{numbers, compress}{natbib}

\usepackage[main, final]{neurips_2026}

\usepackage[utf8]{inputenc} 
\usepackage[T1]{fontenc}    
\usepackage{hyperref}       
\usepackage{url}            
\usepackage{booktabs}       
\usepackage{amsfonts}       
\usepackage{nicefrac}       
\usepackage{microtype}      
\usepackage{xcolor}         

\usepackage{graphicx}
\usepackage{subfigure}
\usepackage{makecell}
\usepackage{pifont}      
\usepackage{grffile}
\usepackage{colortbl}
\usepackage{tabularx}
\usepackage{bbding}
\usepackage{amsmath}
 \usepackage{multirow}
\usepackage[ruled,vlined]{algorithm2e}
\usepackage{mathtools}
\usepackage{amsthm}
\newtheorem{theorem}{Theorem}

\title{SICAF: Time-Varying Focus Bottleneck for Self-Supervised Event-Based Optical Flow with Spiking Neural Network}

\author{%
  Shuangming Yang\thanks{Corresponding authors: Shuangming Yang, Shangqi Guo, and Badong Chen.} \\
  Tianjin University \\
  Tianjin, China \\
  \texttt{yangshuangming@tju.edu.cn} \\
  \And
  Qing He \\
  Tianjin University \\
  Tianjin, China \\
  \texttt{3020234423@tju.edu.cn} \\
  \And
  Shangqi Guo\textsuperscript{*} \\
  Tsinghua University \\
  Beijing, China \\
  \texttt{shangqi\_guo@mail.tsinghua.edu.cn} \\
  \And
  Badong Chen\textsuperscript{*} \\
  Xi'an Jiaotong University \\
  Xi'an, China \\
  \texttt{chenbd@mail.xjtu.edu.cn} \\
}

\begin{document}

\maketitle

\begin{abstract}
  Event-based optical flow estimation achieves high temporal resolution and dynamic range by exploiting asynchronous events, while self-supervised learning provides an effective solution to the absence of ground-truth labels. Spiking neural networks (SNNs), as brain-inspired models, process information through asynchronous spikes, naturally aligning with optical flow events. However, the temporal irregularity across event streams hinders the compact extraction of spatiotemporal features, while the absence of explicit labels leads to spurious correlations, thereby increasing redundancy within latent representations. In this paper, we propose a novel  “Compactly Compress, Time-varying Focus” training strategy. We present SICAF (\textbf{S}pikes \textbf{I}nformation \textbf{C}ompression \textbf{A}nd \textbf{F}ocus) framework that formulates the learning process as a constrained optimization with temporal modulation. It encourages self-supervised SNNs to extract compact latent representations while preventing over-compression that would discard temporal saliency. To better enable feature selectivity across time steps, we design an estimator with learnable temporal attention, achieving context awareness with cross-time dependencies. Experimental results on several datasets demonstrate that SICAF achieves state-of-the-art performance in self-supervised SNNs, with a 8.14\% improvement in AEE on the MVSEC dataset and a 13.49\% increase in robustness under white-box adversarial attacks.
\end{abstract}

\section{Introduction}
\label{sec:intro}

Event-based optical flow estimation is an advanced technique that utilizes event cameras to measure motion in visual scenes \cite{cuadrado2023optical, ding2022spatio, paredesvalles2023taming}. With high temporal resolution and low power consumption, event cameras outperform frame-based sensors by delivering precise, rapid motion detection, enabling energy-efficient optical flow estimation in high-speed and low-light environments \cite{gallego2018unifying, schneider2023neuromorphic}.  Since ground-truth optical flow labels are difficult to obtain, recent work has adopted self-supervised learning for event-based optical flow without annotation \cite{hagenaars2021self, kosta2023adaptive, lee2020spikeflownet, sun2025stflownet}.

Emerging as a promising candidate for neuromorphic computing, spiking neural networks (SNNs) employ temporal coding mechanisms to efficiently encode time-dependent dynamic information \cite{dampfhoffer2023backpropagation, gallego2020event}. Besides, SNNs operate in an event-driven spiking manner, allowing flexible processing of asynchronous and sparse data \cite{barchid2023spiking, qiao2024spatio}. These properties make SNNs highly promising for handling event-based optical flow \cite{cuadrado2023optical, paredesvalles2019unsupervised}. Recent work on SNNs has successfully developed self-supervised learning strategies for event-based optical flow estimation \cite{kosta2023adaptive, hagenaars2021self, apolinario2023stllr}, and further incorporated learnable membrane-potential parameters to improve their representational capacity \cite{kosta2023adaptive}. Yang \emph{et al.} \cite{yang2024self} propose SeLHIB, a self-supervised high-order information bottleneck training scheme, which improves the robustness of event-based optical flow estimation.

Despite these significant advances, a glaring limitation persists: variations in motion patterns give rise to event streams that are highly irregular along the temporal dimension, while existing optical flow learning methods largely overlook such temporal heterogeneity. This omission hinders the model's ability to capture informative motion cues in the event stream, especially in complex scenarios involving fast motion, adverse illumination conditions, or highly dynamic backgrounds. This brings us to an intriguing inquiry: \emph{Is it feasible to design an SNN learning strategy that intrinsically accounts for the temporal evolution of spike-based latent representations when extracting compact yet informative motion features?}

\begin{figure}[htbp]
    \centering
    \includegraphics[width=0.8\linewidth]{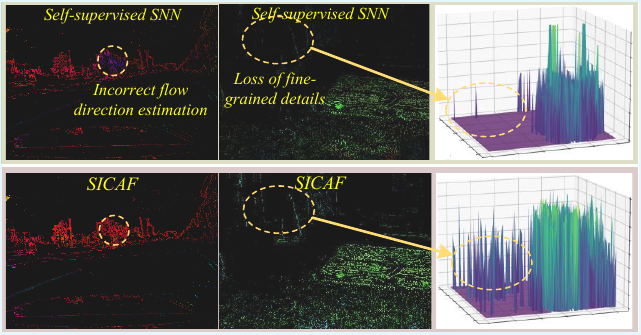}
    \caption{\textbf{Visualization of performance evaluation with previous self-supervised SNN (top) and the proposed SICAF (bottom).} Our SICAF enhances the estimation accuracy by emphasizing temporal saliency. Concurrently, our SICAF effectively captures structured yet informative motion features, improving the issues of incorrect flow direction estimation and the loss of fine-grained motion details.}
    \label{figure1}
\end{figure}

In response to the aforementioned challenge, we propose a pioneering training strategy termed \emph{\textbf{Compactly Compress, Time-varying Focus}}. Central to our approach is a novel adaptive focus mechanism equipped with a flexible information compression strategy across time steps to achieve temporal feature selectivity. While learning compact latent representations, the proposed approach explicitly takes into account the temporal heterogeneity of the input event streams. As illustrated in Fig.~\ref{figure1}, previous self-supervised SNNs often suffer from incorrect optical flow direction estimation and the loss of fine-grained motion details, whereas our approach strengthens the capability of SNNs to represent motion patterns with more accurate optical flow predictions.

In event-based optical flow estimation, general self-supervised SNNs typically adopt contrast maximization, as illustrated in Fig.~\ref{figure2}(a). First, such self-supervised SNNs lack a mechanism for compressing redundant and noisy information, and therefore fail to extract compact, task-relevant spatiotemporal representations from event data. Second, contrast maximization drives learning solely based on the spatial alignment of events, without exploiting their timestamps or temporal distribution. SeLHIB \cite{yang2024self}, as shown in Fig.~\ref{figure2}(b), captures more compact representations by formulating information bottleneck (IB) constraint. Nevertheless, this form of IB fails to consider the irregular temporal distribution of event streams. A uniform information bottleneck applied across time is inadequate  for capturing temporally salient features. As a result, jointly handling temporal non-uniformity and redundancy compression remains somewhat uncharted.

\begin{figure*}[htbp]
    \centering
    \includegraphics[width=1\linewidth]{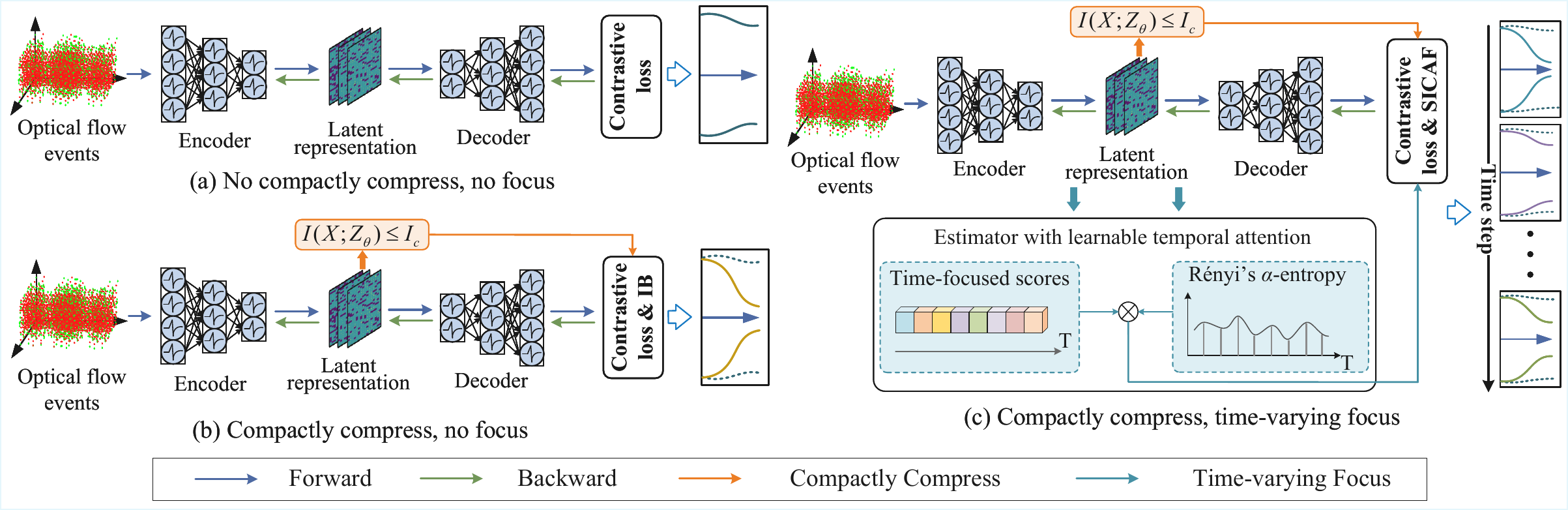}
    \caption{\textbf{Comparison between the general self-supervised SNNs \cite{kosta2023adaptive, hagenaars2021self, apolinario2023stllr}, SeLHIB \cite{yang2024self}, and ours for handling optical flow events with self-supervised learning strategy.} (a) General self-supervised SNNs typically rely on contrast maximization. This objective enhances the spatial contrast of motion-compensated events to drive learning, but it neglects the temporal non-uniformity of event data. (b) SeLHIB presents an information bottleneck constraint to encourage compact latent representations. However, because it enforces a uniform compression constraint across time, it struggles to preserve temporally salient features. (c) \textbf{SICAF} addresses these issues by formulating a constrained optimization process with temporally modulated focus. Within this framework, a learnable temporal-attention–based estimator allows SICAF to explicitly model temporal heterogeneity, encouraging the formation of compact yet more temporally informative representations.}
    \label{figure2}
\end{figure*}

Targeting the above challenge, we propose \underline{\textbf{S}}pikes \textbf{\underline{I}}nformation \textbf{\underline{C}}ompression \textbf{\underline{A}}nd \textbf{\underline{F}}ocus (\textbf{SICAF}), which formulates the learning process as a constrained optimization with temporal modulation. As depicted in Fig.~\ref{figure2}(c), SICAF learns a latent and compact space in a time-varying compression manner while capturing temporal salient feature across time steps. To this end,  we first design a novel estimator with learnable temporal attention to modulate the compression process. Specifically, to capture temporal dynamics in the representation learning process, we map the spike firing rates of latent representations into queries and keys to compute temporal attention scores. This enables the model to capture feature variations along the temporal dimension and provides a time-varying focusing capability for the input event stream. Building on Rényi's $\alpha$-entropy, the estimator robustly measures the entropy of high-dimensional latent representations at different time steps. These entropy values, weighted by the temporal attention scores, form a time-varying focusing constraint that modulates the strength of information compression. Extensive experiments demonstrate that SICAF achieves the top performance in self-supervised SNNs. Overall, the main \textbf{contributions} of this paper are:
\begin{itemize}
\item We present a novel "compactly compress, time-varying focus" paradigm,  which learns to enhance the temporal saliency of representations while suppressing redundant event information.  
\item We propose SICAF, a constrained optimization framework with temporal modulation by neural activities.  It  enjoys a concise and tractable objective and enables adaptive, time-varying compression of irregular optical-flow events.
\item We design a learnable temporal-attention estimator that directly extracts fine-grained temporal cues from spiking neural activities and dynamically guides the compression process.   
\item Experiments demonstrate that SICAF achieves superior performance compared with existing state-of-the-art SNN-based optical flow models, while significantly enhancing the robustness. SICAF can serve as a strong baseline and inspire further advancements in the field of neuromorphic computing and brain-inspired intelligence.
\end{itemize}

\section{Related Work}
\label{sec:related work}

\subsection{Event-based Optical Flow Estimation with SNNs}
Event-based optical flow estimation with SNNs has attracted growing interest. Early work exploited adaptive neurons and spike timing-dependent plasticity (STDP) to develop unsupervised SNN training for efficient high-bandwidth flow estimation \cite{paredesvalles2019unsupervised}. Hybrid architectures such as Spike-FlowNet further showed that SNN encoders can match artificial neural network (ANN) performance when paired with ANN decoders \cite{lee2020spikeflownet}. Subsequent efforts introduced self-supervised frameworks leveraging temporal cues from short event slices \cite{hagenaars2021self}, multi-sensor fusion with frame-based inputs \cite{lee2022}, and adaptive neural mechanisms for improved spatiotemporal extraction \cite{kosta2023adaptive}. S-TLLR proposed a three-factor temporal local learning rule inspired by STDP, achieving state-of-the-art performance under self-supervision \cite{Apolinario2023}. ANN-to-SNN conversion approaches have also been explored to enhance estimation accuracy \cite{sun2025stflownet}. Directly trained SNNs, however, offer unique advantages by fully exploiting the intrinsic temporal dynamics of spiking neurons, enabling low-latency and efficient spatiotemporal feature learning \cite{Wu2019, Zheng2021}. In this work, we follow this direction and focus on direct training of SNNs to unlock their full potential for event-based optical flow estimation.

\subsection{Direct Training of SNNs with the IB Principle}
Recent advances have applied the IB principle to directly train SNNs. Methods such as HOSIB and HHO-IB introduce higher-order IB formulations for surrogate-gradient learning, yielding notable gains in generalization and robustness \cite{yang2023effective, Wu2025}. SeLHIB further presents an IB-based self-supervised framework for event-based optical flow with SNNs, substantially improving robustness under noisy conditions \cite{yang2025selfsupervised}. However, this uniform temporal compression suppresses critical temporal cues. We address this limitation by proposing a strategy that preserves compactness while selectively retaining temporally salient features.

\section{Preliminary}
\label{Preliminary}

\paragraph{Spiking Neuron Model.} Spiking neurons are the core computational units of SNNs. They integrate incoming postsynaptic potentials, emit a spike once the membrane potential crosses a firing threshold, and thereby enable efficient event-driven processing of asynchronous data. The mathematical modeling is presented below:
\begin{equation}
u^{(t)}_{\textit{pre}} = \tau u^{(t-1)} + i^{(t)},
\end{equation}
\begin{equation}
u^{(t)} = u^{(t)}_{\textit{pre}} (1 - o^{(t)}),
\end{equation}
where $ u^{(t)}_{\textit{pre}} $ and $ u^{(t)} $ denote the pre-update and post-update membrane potentials in the leaky integrate-and-fire (LIF) model, respectively. $ i^{(t)} $ represents the synaptic input current at timestep $ t $, while $ o^{(t)} \in \{0,1\} $ indicates the binary output spike generated by the spiking neuron. The detailed neuron parameter settings are provided in Appendix~\ref{sec:rationale}.



\paragraph{Contrast Maximization.} 
Contrast maximization is a self-supervised approach for optical flow estimation that exploits the spatiotemporal misalignment of events in a continuous stream~\cite{gallego2019focus}. By applying motion compensation, events are warped to a common reference frame, reducing motion-induced blur. Assuming \( u(x) \) denotes the optical flow at each pixel location and \( \bar{t}(x, y; u \mid t_{\text{ref}}) \) the average timestamp at each pixel. 
The objective is to minimize the normalized sum of squared average timestamps, yielding the following convex contrastive loss:
\begin{equation}
\mathcal{L}_{\text{contrast}}(t_{\text{ref}}) = \frac{
\sum_x \bar{t}_{p=+}(x; u \mid t_{\text{ref}})^2 + \bar{t}_{p=-}(x; u \mid t_{\text{ref}})^2
}{
\sum_x \left[ n(x') > 0 \right] + \epsilon
}.
\end{equation}

\section{Method}

Fig.~\ref{figure2}(c) visually elucidates the overall procedure of SICAF.  After being encoded through $E(Z|X)$, the input $X$ is mapped to the latent representation $Z$, and the upper bound of the mutual information between $X$ and $Z$, denoted $I(X;Z) \le I_c$, enforces compact compression during the feature extraction process. We compute time-focused scores from the spike-rate matrices of $Z$ along the temporal dimension, enabling contextual awareness across time dependencies. Moreover, we employ the Rényi’s $\alpha$-entropy estimator to robustly quantify the information content of the latent variables. The entropy is multiplied by the time-focused scores, forming an estimator with learnable temporal attention that maximizes the entropy estimation to preserve the temporal heterogeneity of event streams. The pseudocode of the method is provided in Appendix~\ref{pseudocode}, and the generalization bound as well as the sample complexit bound theorems and their proofs are presented in Appendix\ref{app:theory}.

\subsection{Temporal Attention Scores}
\label{sec4.1}

Unlike conventional ANNs, SNNs communicate through discrete spike events, allowing temporal information to be encoded directly in spike timing as activity propagates across layers. Motivated by this inherent temporal structure, we compute attention weights over latent representations across time, forming a temporal attention matrix  $T_A$. Specifically, for a latent representation $Z \in \mathbb{R}^{T \times C \times D}$, we first compute its spike rate matrix $T_Z \in \mathbb{R}^{T \times C}$, and then project it into two latent spaces as follows:
\begin{equation}
    Z_Q = W_Q \cdot T_Z, \quad 
    Z_K = W_K \cdot T_Z
\end{equation}
where $Z_Q,Z_K\in\mathbb{R}^{T\times T}$ represent temporal queries and keys, respectively. The temporal attention matrix is computed as:
\begin{equation}
A = \operatorname{softmax}_{\mathrm{dim}=2}
\left(\frac{Z_Q \cdot Z_K^{T}}{\sqrt{d_k}}\right),
\end{equation}
where $d_k$ is the key dimension. The softmax is applied along the key temporal dimension (the second dimension of the attention matrix), such that $\sum_{j=1}^{T} A_{ij}=1$ for each query timestep $i$.We then average the attention weights over all query timesteps to obtain the normalized temporal attention distribution:
\begin{equation}
T_{A,j} = \frac{1}{T}\sum_{i=1}^{T} A_{ij},
\qquad j=1,\ldots,T.
\end{equation}
where $T_A\in\mathbb{R}^{T\times 1}$ assigns an attention weight to each timestep and satisfies $T_{A,j}\geq 0$ and $\sum_{j=1}^{T}T_{A,j}=1$. Since event-based streams exhibit sparse and irregular temporal distributions, $T_A$ explicitly captures this variability and provides a principled basis for the temporally selective accumulation mechanism used in SICAF.

\begin{figure}[htbp]
    \centering
    \includegraphics[width=0.6\linewidth]{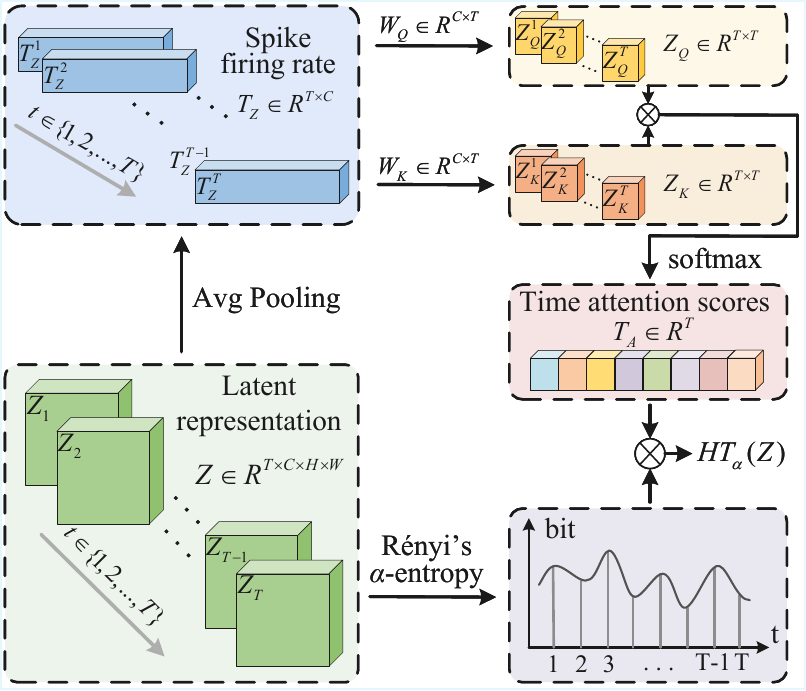}
    \caption{\textbf{Overview of the time-sensitive Rényi’s $\alpha$-entropy estimator.}  Temporal attention scores are derived from spike firing rates of latent representations to capture temporal heterogeneity. Rényi’s $\alpha$-entropy provides a robust and flexible measure of information content, enabling selective feature compression when combined with temporal attention.}
    \label{figure3}
\end{figure}

\subsection{Time-Sensitive Rényi's $\alpha$-Entropy Estimator}
\label{sec4.2}

We construct a time-sensitive Rényi's $\alpha$-entropy estimator $HT_{\alpha}(Z)$, as illustrated in Fig~\ref{figure3}, and propose a spike-rate-based temporal selection information bottleneck. In information theory, Rényi's $\alpha$-entropy can measure the amount of information contained in a variable while exhibiting better numerical stability and scalability compared to Shannon entropy. It thus provides a robust and flexible theoretical tool for estimating high-dimensional temporal information features.

Given a kernel function $k: \chi \times \chi \rightarrow \mathbb{R}$, the Gram matrix $K$ in the latent variable $Z$ can be calculated as:
$K_{ij} = k(Z_i, Z_j)$, and the normalized positive semi-definite matrix $\Phi$ is computed as: $ \Phi_{ij} = \frac{K_{ij}}{\sqrt{K_{ii}K_{jj}}} $.

For the latent variable $Z$, its Rényi's $\alpha$-entropy is given by:
\begin{equation}
\begin{aligned}
H_{\alpha}(Z) &= \frac{1}{1-\alpha}\log_2\left(\mathrm{tr}(\Phi^{\alpha})\right) \\
&= \frac{1}{1-\alpha}\log_2\left(\sum_{i=1}^{N}\lambda_i(\Phi)^{\alpha}\right),
\end{aligned}
\end{equation}
where $\lambda_i$ denotes the $i_{th}$ eigenvalue of the latent representation $Z$. In the subsection \ref{sec4.1}, we compute the temporal attention scores based on spike-rate temporal information, thereby constructing a time-sensitive Rényi's $\alpha$ entropy estimator 
\begin{equation}
HT_{\alpha}(Z) = \sum_{j=1}^{T} T_{Aj} H_{\alpha}(Z_j),
\end{equation}
This estimator flexibly evaluates the information content of the latent variable across the temporal dimension, adaptively assigning different attention weights. Thus, in the entropy estimation process, it preserves the temporal heterogeneity of the latent variable.

\subsection{SICAF Learning Strategy}

The IB framework achieves redundant information compression by minimizing the mutual information $I(Z; X)$ between the input $X$ and the latent representation $Z$. However, most existing SNN-based IB approaches fail to consider the intrinsic temporal characteristics of SNNs, leading to limited advantages in temporal sequence tasks such as optical flow estimation. To address this limitation, we construct a time-sensitive Rényi's $\alpha$-entropy estimator and propose a spike-based temporal attention information bottleneck learning strategy.  

We introduce $Y$ as an abstract task-output variable for the theoretical formulation; it denotes neither available ground-truth optical flow labels nor explicitly generated pixel-wise pseudo-labels.For input $X$ and output $Y$, we use an encoder $E(Z|X)$ that projects the input $X$ into the latent space to obtain a representation $Z$. The objective of SICAF is to realize \textbf{Compactly Compress and Time-varying Focus} during this mapping process. Specifically, \textbf{Compactly Compress} imposes an upper-bound constraint on the mutual information between $X$ and $Z$: $I(X; Z) \le I_c,$ while \textbf{Time-varying Focus} enforces a temporal modulation constraint during compression, expressed as: $HT_\alpha(Z) = \sum_{i=1}^{T} A_i H_\alpha(Z_i) \ge H_c, $
where the learning process is described as an optimization problem with temporal modulation constraints. 

Assuming that the mapping between the latent representation $Z$ and the output $Y$ is defined by $q(y|z)$, 
the optimization objective can be expressed as:
\[\mathcal{L}(q, E) = \min_{E, q} \; \mathbb{E}_{x,y \sim p(x,y)} \big[ \mathbb{E}_{z \sim E(z|x)}[-\log q(y|z)] \big]\]
\begin{equation}
s.t.\quad I(X; Z) \le I_c, \quad HT_{\alpha}(Z) \ge H_c,
\end{equation}
where the computation of $HT_{\alpha}(Z)$ is detailed in subsection~\ref{sec4.2}. The mutual information $I(X; Z)$ is defined as:
\begin{equation}
\begin{aligned}
I(X; Z) &= \int p(x, z) \log \frac{p(x, z)}{p(x)p(z)} \, dx dz \\ 
        &= \int p(x) \mathbb{E}_{E(z|x)} \log \frac{E(z|x)}{p(z)} \, dx dz,
\end{aligned}
\end{equation}
where $p(x)$ and $p(z)$ denote the marginal distributions of the input variable and the latent variable, respectively. Since the distribution $p(z)$ of the latent variable is often intractable, we approximate it with a variational distribution $r(z)$ that satisfies the KL divergence constraint $\mathrm{KL}[p(z)\|r(z)] \ge 0$. 
Hence,
\begin{equation}
\int p(z) \log p(z) \, dz \ge \int p(z) \log r(z) \, dz.
\end{equation}

Using this inequality, we obtain an upper bound for the mutual information $I(X; Z)$:
\begin{equation}
\begin{aligned}
I(X; Z) &\le \int p(x) \mathbb{E}_{E(z|x)} \log \frac{E(z|x)}{r(z)} \, dx dz \\
        &= \mathbb{E}_{x \sim p(x)}[\mathrm{KL}[E(z|x)\|r(z)]].
\end{aligned}
\end{equation}

Therefore, the constrained optimization objective can be approximated by the following upper bound:
\[\widetilde{\mathcal{L}}(q, E) = \min_{E, q} \; 
\mathbb{E}_{x,y \sim p(x,y)} \big[ \mathbb{E}_{z \sim E(z|x)}[-\log q(y|z)] \big]\]
\begin{equation}
\begin{aligned}
s.t. &\quad \mathbb{E}_{x \sim p(x)}[\mathrm{KL}[E(z|x)\|r(z)]] \le I_c, \\
&\quad HT_{\alpha}(Z) \ge H_c.
\end{aligned}
\end{equation}

For the model training process, 
we introduce regularization coefficients $\beta$ and $\gamma$ to transform the constrained optimization into an unconstrained training objective:
\begin{equation}
\begin{aligned}
J_{\text{SICAF}} 
= & \; \min_{E, q} \; 
\mathbb{E}_{x,y \sim p(x,y)} \big[ \mathbb{E}_{z \sim E(z|x)}[-\log q(y|z)] \big] \\
& + \beta \cdot \big( \mathbb{E}_{x \sim p(x)}[\mathrm{KL}[E(z|x)\|r(z)]] - I_c \big) \\
&- \gamma \cdot \big( HT_{\alpha}(Z) - H_c \big).
\end{aligned}
\end{equation}

\begin{figure*}[htbp]
    \centering
    \includegraphics[width=1\linewidth]{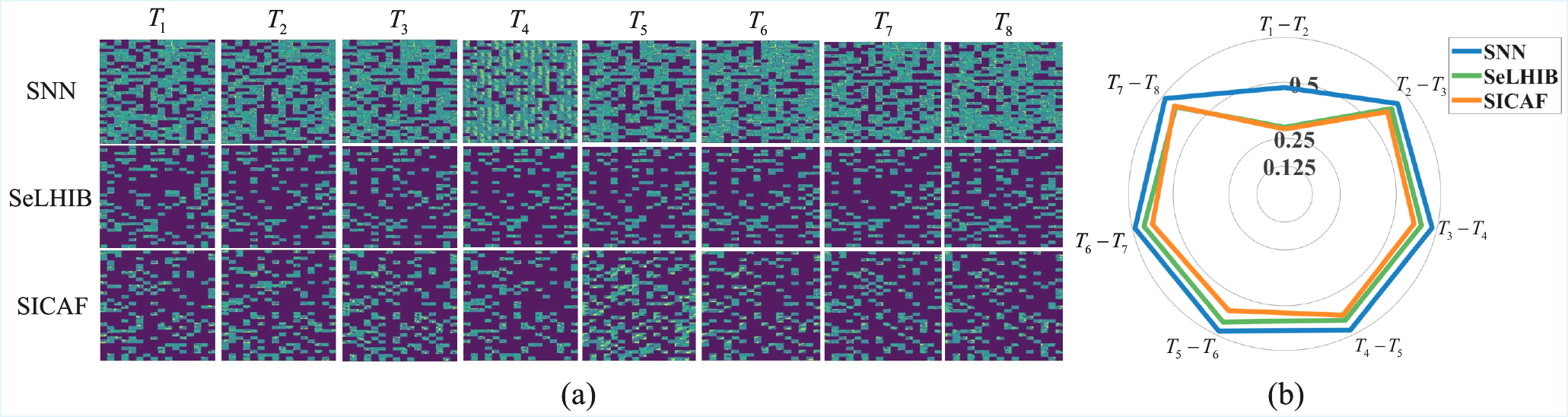}
    \caption{(a) Comparison of feature maps in latent space. Both SeLHIB and our SICAF demonstrates sparser feature maps compared to baseline self-supervised SNN, indicating the acquisition of more structured spatiotemporal features. (b) The radar charts of SSIM for self-supervised SNN, SeLHIB, and SICAF, respectively.}
    \label{figure4}
\end{figure*}

\subsection{Mechanistic Analysis of Temporal Heterogeneity in SICAF}

To evaluate the effectiveness of SICAF in achieving compact compression and time-varying focusing for irregular event streams, we visualize the latent feature maps of trained models and measure the structural similarity of their neuronal firing patterns across consecutive time steps. As shown in Fig.~\ref{figure4}(a), the SNN baseline exhibits the densest activations, whereas both SeLHIB and SICAF produce noticeably sparser patterns. This suggests that both of these two strategies encourage more compact and better-structured feature representations. 

To more intuitively examine how the firing patterns of the three methods vary across time steps, we plot radar charts of structural similarity (SSIM) to measure the similarity between adjacent time steps. The SSIM is defined as:
\begin{equation}
\mathrm{SSIM}(x, y) = 
\frac{(2\mu_x\mu_y + C_1)(2\sigma_{xy} + C_2)}
     {(\mu_x^2 + \mu_y^2 + C_1)(\sigma_x^2 + \sigma_y^2 + C_2)},
\end{equation}
where $\mu_x$, $\sigma_x^2$, $\mu_y$, and $\sigma_y^2$ denote the mean and variance of images $x$ and $y$, $\sigma_{xy}$ is their covariance, and $C_1$ and $C_2$ are stabilization constants.
As shown in Fig.~\ref{figure4}(b), SICAF achieves consistently lower SSIM scores than both the conventional self-supervised SNN and the SeLHIB baseline. Because SSIM measures structural similarity between temporally adjacent firing maps, lower SSIM indicates stronger temporal decorrelation, implying that the network produces activation patterns with higher temporal resolution. 
In essence, the lower SSIM values observed in SICAF reflect a theoretically grounded effect: SICAF enforces temporal diversity during minimizing redundant dependencies, thereby promoting a temporally selectivity regime that aligns with the irregular statistics of event-based optical flow.

\section{Experiments}
\label{Experiments}

\subsection{Experimental Setup}
We demonstrate the effectiveness of SICAF on event-based optical flow estimation task with different datasets, including multi-vehicle stereo event-camera (MVSEC), high quality frames (HQF) \cite{stoffregen2020reducing}, and event camera dataset (ECD) \cite{mueggler2017event}. All experiments were conducted using the PyTorch framework on RTX 4090 and A6000 GPUs. Adam optimizer is used with learning rate of 0.0002 and batch size of 8 for 50 epochs. Nearly 15 minutes of event data are contained in the training sequences, and they are split into 140 sequences with 128×128 size with 500k events in each sequence. Optical flow predictions are produced at each frame timestep with \textit{dt}=2, i.e., 25\% of the window events at a time. 

\subsection{Performance Evaluation on Event-based Optical Flow Estimation}

We evaluated the generalization ability of SICAF in the self-supervised optical flow estimation task across the MVSEC, ECD, and HQF datasets, and compared it with the SOTA methods. The results on the MVSEC dataset are shown in Table~\ref{table1}, while the results on the ECD and HQF datasets are shown in Appendix~\ref{OFER}. The results demonstrate that SICAF significantly improves the accuracy of SNN in self-supervised optical flow estimation tasks. Specifically, compared to the work that applies IB to the SNN self-supervised optical flow task \cite{yang2025selfsupervised}, SICAF achieves a 8.16\% improvement, which suggests that the integration of compact compression and time-varying focusing enables SICAF to achieve compact yet selective feature extraction, improving the model’s generalization ability. Furthermore, we compared the qualitative optical flow estimation results of SICAF and the method without SICAF, as shown in Figure~\ref{Qualitative evaluation1}. From the comparison of the flow estimation results, it is noticeable that SICAF provides more accurate and consistent motion direction predictions, while preserving finer motion details and producing clearer and more complete motion contours.

\begin{table}[htbp]
    \centering
    \caption{Experimental results of SICAF on MVSEC compared to previous work. (lower is better, best in red, runner-up in blue).}
    \label{table1}
    \footnotesize
    \begin{tabular}{lccccc}
        \toprule
        \multirow{3.5}{*}{Methods} & \multicolumn{5}{c}{Average Endpoint Error (AEE) $\downarrow$} \\ 
        \cmidrule{2-6}
        & \begin{tabular}[c]{@{}c@{}}Indoor\\ (flying1)\end{tabular} 
        & \begin{tabular}[c]{@{}c@{}}Indoor\\ (flying2)\end{tabular} 
        & \begin{tabular}[c]{@{}c@{}}Indoor\\ (flying3)\end{tabular} 
        & \begin{tabular}[c]{@{}c@{}}Outdoor\\ (day1)\end{tabular} 
        & Avg. \\ 
        \midrule
        LIF-EV-FlowNet \cite{hagenaars2021self} & \color{red} 0.71 & 1.44 & 1.16 & 0.53 & 0.96 \\
        LIF-FireNet \cite{hagenaars2021self} & 0.98 & 1.77 & 1.50 & 0.57 & 1.21 \\
        ALIF-EV-FlowNet \cite{hagenaars2021self} & 1.00 & 1.78 & 1.55 & 0.57 & 1.23 \\
        PLIF-EV-FlowNet \cite{hagenaars2021self} & 0.75 & 1.52 & 1.23 & 0.60 & 1.03 \\
        XLIF-EV-FlowNet \cite{hagenaars2021self} & \color{blue} 0.73 & 1.45 & 1.17 & 0.45 & 0.95 \\
        ALIF-FireNet \cite{hagenaars2021self} & 1.04 & 1.85 & 1.58 & 0.62 & 1.27 \\
        PLIF-FireNet \cite{hagenaars2021self} & 0.90 & 1.67 & 1.41 & 0.56 & 1.14 \\
        XLIF-FireNet \cite{hagenaars2021self} & 0.98 & 1.82 & 1.54 & 0.54 & 1.22 \\
        Adaptive-SpikeNet \cite{kosta2023adaptive} & 0.84 & 1.59 & 1.36 & 0.48 & 1.07 \\
        Adaptive-SpikeNet \cite{kosta2023adaptive} & 0.79 & 1.37 & 1.11 & 0.44 & 0.93 \\
        SeLHIB \cite{yang2025selfsupervised} & 0.74 & 1.48 & 1.21 & 0.50 & 0.98 \\
        FSFN \cite{apolinario2023stllr} & 0.76 & \color{blue} 1.19 & \color{blue} 1.00 & 0.50 & \color{blue} 0.86 \\
        \rowcolor[HTML]{EAF5FB}
        \textbf{EV-FlowNet + SICAF} & \color{red} 0.71 & 1.37 & 1.09 & \color{blue} 0.43 & 0.90 \\
        \rowcolor[HTML]{EAF5FB}
        \textbf{FSFN + SICAF} & \color{blue} 0.73 & \color{red} 1.09 & \color{red} 0.93 & \color{red} 0.42 & \color{red} 0.79 \\
        \bottomrule
    \end{tabular}
\end{table}

\begin{figure}[htbp]
    \centering
    \includegraphics[width=0.8\linewidth]{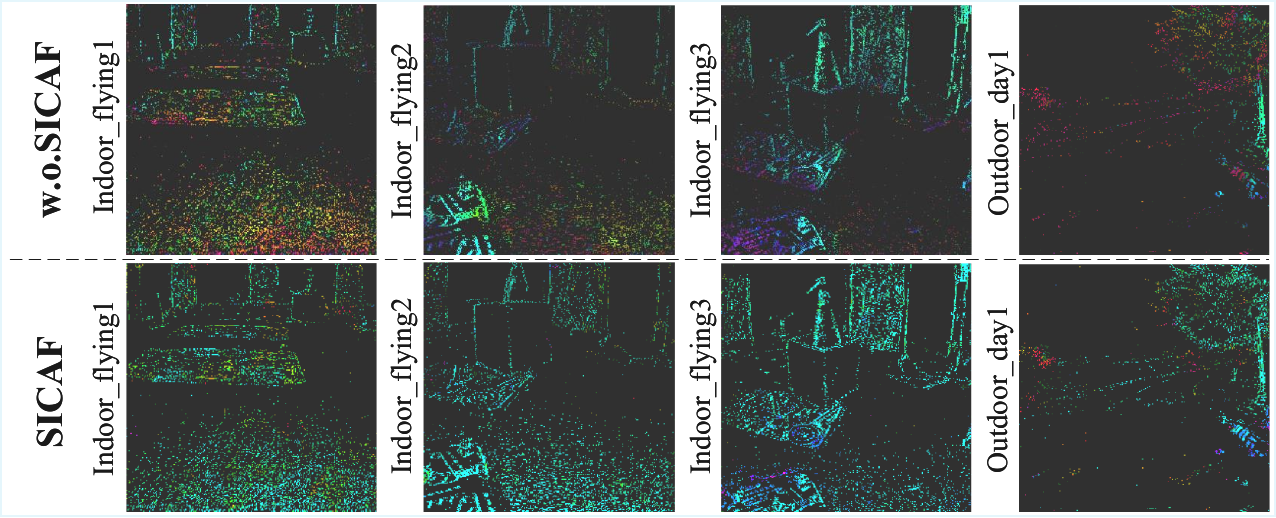}
    \caption{Qualitative evaluation of optical flow estimation results on sequences from the MVSEC dataset.}
    \label{Qualitative evaluation1}
\end{figure}


\subsection{Robustness of Event-based Optical Flow Estimation in Noisy Situations}

Due to environmental interference, event cameras often generate spurious events caused by non-physical brightness changes. These noisy events can interfere with the model’s ability to extract accurate motion features.
To further investigate the impact of SICAF on the robustness of SNNs in self-supervised optical flow estimation, we artificially injected various types of noise into the input event streams and compared the results with a baseline model, as shown in Table \ref{robust}.
The results demonstrate that SICAF significantly enhances robustness under both random and adversarial noise, achieving a 13.49\% improvement in accuracy under white-box adversarial conditions.
Additional experimental results on the ECD and HQF datasets are provided in Appendix~\ref{robustness_supp}.
SICAF enhances robustness under white-box adversarial attacks through its time-sensitive adaptive compression mechanism. The time-varying compression framework enables SICAF to dynamically focus on critical events while suppressing irrelevant or noisy data, leading to more robust and structured feature extraction and thus making it less susceptible to adversarial perturbations. It also adjusts the weighting of temporal features, further aids in capturing temporal heterogeneity and reducing the impact of adversarial noise.

\begin{table*}[ht]
\centering
\caption{ Experimental results of SICAF on MVSEC under different noise (lower is better, best in red, runner-up in blue).}
\label{robust}
\resizebox{0.99\linewidth}{!}{%
\begin{tabular}{ccccccc}
\hline
\multirow{2}{*}{Noise type}       & \multirow{2}{*}{Methods} & \multicolumn{5}{c}{AEE}                                                                 \\  \cline{3-7} 
                                  &                          & Indoor (flying1) & Indoor (flying2) & Indoor (flying3) & Outdoor (day1) & Average       \\ \hline
\multirow{4}{*}{SP}               & LIF-EV-FlowNet \cite{hagenaars2021self}                & 0.91             & 1.68             & 1.39             & 0.56           & 1.14          \\
                                  & XLIF-EV-FlowNet \cite{hagenaars2021self} & 0.86 & 1.76 & 1.21 & 0.62 &1.11
                                  \\
                                  & \cellcolor[HTML]{EAF5FB}SICAF-LIF                 & \cellcolor[HTML]{EAF5FB}\color{blue} 0.82       & \cellcolor[HTML]{EAF5FB}\color{blue} 1.44       & \cellcolor[HTML]{EAF5FB}\color{blue} 1.20       & \cellcolor[HTML]{EAF5FB}\color{blue} 0.60     & \cellcolor[HTML]{EAF5FB}\color{blue} 1.01    \\
                                  & \cellcolor[HTML]{EAF5FB}SICAF-XLIF                & \cellcolor[HTML]{EAF5FB}\color{red} 0.75    & \cellcolor[HTML]{EAF5FB}\color{red} 1.43    & \cellcolor[HTML]{EAF5FB}\color{red} 1.18    & \cellcolor[HTML]{EAF5FB}\color{red} 0.44  & \cellcolor[HTML]{EAF5FB}\color{red} 0.95 \\ \hline
\multirow{4}{*}{Gaussian}         & LIF-EV-FlowNet \cite{hagenaars2021self}               & 0.95             & 1.72             & 1.44             & 0.62           & 1.18          \\
                                  & XLIF-EV-FlowNet \cite{hagenaars2021self} & \color{blue} 0.93 & 1.70 & 1.39 & 0.62 & 1.16
                                  \\
                                  & \cellcolor[HTML]{EAF5FB}SICAF-LIF                 & \cellcolor[HTML]{EAF5FB}\color{red} 0.83    & \cellcolor[HTML]{EAF5FB}\color{red} 1.49    & \cellcolor[HTML]{EAF5FB}\color{red} 1.25    & \cellcolor[HTML]{EAF5FB}\color{blue} 0.59     & \cellcolor[HTML]{EAF5FB}\color{red} 1.04 \\
                                  & \cellcolor[HTML]{EAF5FB}SICAF-XLIF                & \cellcolor[HTML]{EAF5FB}0.94       & \cellcolor[HTML]{EAF5FB}\color{blue} 1.68       & \cellcolor[HTML]{EAF5FB}\color{blue} 1.42       & \cellcolor[HTML]{EAF5FB}\color{red} 0.53  & \cellcolor[HTML]{EAF5FB}\color{blue} 1.14    \\ \hline
\multirow{4}{*}{Black-box attack} & LIF-EV-FlowNet \cite{hagenaars2021self}               & 1.15             & 1.96             & 1.72             & 0.79           & 1.41          \\
                                  & XLIF-EV-FlowNet \cite{hagenaars2021self} & \color{red} 0.93 & \color{red} 1.79 & 1.78 & 0.73 & 1.31
                                  \\
                                  & \cellcolor[HTML]{EAF5FB}SICAF-LIF                 & \cellcolor[HTML]{EAF5FB}\color{blue} 1.03    & \cellcolor[HTML]{EAF5FB}\color{blue} 1.80    & \cellcolor[HTML]{EAF5FB}\color{red} 1.54    & \cellcolor[HTML]{EAF5FB}\color{red} 0.69  & \cellcolor[HTML]{EAF5FB}\color{red} 1.27 \\
                                  & \cellcolor[HTML]{EAF5FB}SICAF-XLIF                & \cellcolor[HTML]{EAF5FB}1.06       & \cellcolor[HTML]{EAF5FB}1.82       & \cellcolor[HTML]{EAF5FB}\color{blue} 1.58       & \cellcolor[HTML]{EAF5FB}\color{blue} 0.70     & \cellcolor[HTML]{EAF5FB}\color{blue} 1.29    \\ \hline
\multirow{4}{*}{White-box attack} & LIF-EV-FlowNet \cite{hagenaars2021self}               & 1.02             & 1.77             & 1.52             & 0.68           & 1.26          \\
                                  & XLIF-EV-FlowNet \cite{hagenaars2021self} & 0.91 & 1.82 & 1.49 & 0.71 & 1.23
                                  \\
                                  & \cellcolor[HTML]{EAF5FB}SICAF-LIF                 & \cellcolor[HTML]{EAF5FB}\color{red} 0.86    & \cellcolor[HTML]{EAF5FB}\color{red} 1.53    & \cellcolor[HTML]{EAF5FB}\color{red} 1.36    & \cellcolor[HTML]{EAF5FB}\color{red} 0.59  & \cellcolor[HTML]{EAF5FB}\color{red} 1.09 \\
                                  & \cellcolor[HTML]{EAF5FB}SICAF-XLIF                & \cellcolor[HTML]{EAF5FB}\color{blue} 0.90       & \cellcolor[HTML]{EAF5FB}\color{blue} 1.63       & \cellcolor[HTML]{EAF5FB}\color{blue} 1.39       & \cellcolor[HTML]{EAF5FB}\color{blue} 0.58     & \cellcolor[HTML]{EAF5FB}\color{blue} 1.13    \\ \hline
\end{tabular}
}
\end{table*}

\subsection{Ablation on Loss Function Components}

To investigate the roles of various components of loss functions, we compare the performance of \(L_{contrast}\) and SICAF w.o.\(L_{smooth}\), with w.o.SICAF and SICAF on MVSEC. The results in Table~\ref{ablation} show that adding \(L_{smooth}\) to \(L_{contrast}\) improves the estimation accuracy, indicating that it helps denoise after contrast enhancement, reducing the impact of irrelevant information, and thus extracting features more accurately. When SICAF is used without smoothing, the model's accuracy decreases, possibly because high-contrast noise is misinterpreted as valid motion information, leading to feature extraction errors. To mitigate this effect, adding \(L_{smooth}\) is crucial for SICAF.While the above study confirms the necessity of all three loss terms, it does not isolate
the contributions of the internal mechanisms of SICAF itself. To address this, we further conduct a component-level ablation in Section 5.5.

\begin{table}[htbp]
\centering
\caption{Ablation study on the various components of loss functions (lower is better, best in red, runner-up in blue).}
\label{ablation}
\footnotesize
\begin{tabular}{c|c|c|c|ccccc}
\toprule
\multirow{3}{*}{Methods} & \multirow{3}{*}{$L_{contrast}$} & \multicolumn{1}{l|}{\multirow{3}{*}{$L_{smooth}$}} & \multirow{3}{*}{SICAF} & \multicolumn{5}{c}{AEE}                                                                                                                                                                                                                                \\ \cline{5-9} 
                         &                              & \multicolumn{1}{l|}{}                           &                        & \begin{tabular}[c]{@{}c@{}}Indoor\\ (flying1)\end{tabular} & \begin{tabular}[c]{@{}c@{}}Indoor\\ (flying2)\end{tabular} & \begin{tabular}[c]{@{}c@{}}Indoor\\ (flying3)\end{tabular} & \begin{tabular}[c]{@{}c@{}}outdoor\\ (day1)\end{tabular} & Avg. \\ \hline
(a) & \checkmark & \XSolidBrush  & \XSolidBrush  & \color{blue} 0.77  & 1.50  & 1.25 & 0.56  & 1.02 \\
(b)   & \checkmark & \checkmark & \XSolidBrush                      & 0.81  & \color{blue} 1.45  & \color{blue} 1.19  & 0.57  & \color{blue} 1.01 \\
(c)  & \checkmark  & \XSolidBrush  & \checkmark & 0.83  & 1.60  & 1.31  & \color{blue} 0.54 & 1.07 \\
(d)   & \checkmark   & \checkmark  & \checkmark                      & \color{red} \textbf{0.71}  & \color{red} \textbf{1.37}  & \color{red} \textbf{1.09}  & \color{red} \textbf{0.43}  & \color{red} \textbf{0.90} \\ \hline
\end{tabular}
\end{table}


\subsection{Ablation on Internal Components of SICAF}
\label{sec:ablation_components}

To investigate the roles of the key components of SICAF, we compare the performance of SICAF,
SICAF under uniform compression, SICAF w.o. compact compression, and SICAF w.o. time-varying
focus on MVSEC. The results in Table~\ref{tab:ablation_components} show that SICAF achieves
the lowest AEE across all four sequences through its temporally sensitive
information-compression strategy, demonstrating the effectiveness of the proposed method.
 
When the temporally varying compression mechanism is replaced with uniform compression, the
model's accuracy decreases, possibly because a uniform information-compression strategy is
unable to extract compact and critical features effectively from event streams with temporal
heterogeneity, where excessive compression of temporal segments containing important motion
information leads to an increase in prediction error. To mitigate this effect, the temporally
varying compression is crucial for SICAF.
 
Removing either of the two key components of SICAF also degrades the prediction accuracy,
indicating that the compact compression mechanism suppresses redundant information, while the
temporally sensitive entropy constraint $HT_\alpha(Z)$ preserves critical temporal
information. These two components work collaboratively in event-based optical flow estimation
and are both indispensable.

\begin{table}[t]
\centering
\caption{Ablation study on the key components of SICAF on the MVSEC dataset (lower is better,
best in red, runner-up in blue).}
\label{tab:ablation_components}
\footnotesize
\resizebox{0.99\linewidth}{!}{%
\begin{tabular}{c|c|c|c|ccccc}
\toprule
\multirow{3}{*}{Methods} &
\multirow{3}{*}{\shortstack{Compact\\Compression}} &
\multirow{3}{*}{\shortstack{Time-Varying\\Focus}} &
\multirow{3}{*}{\shortstack{Uniform\\Compression}} &
\multicolumn{5}{c}{AEE} \\ \cline{5-9}
& & & &
\begin{tabular}[c]{@{}c@{}}Indoor\\ (flying1)\end{tabular} &
\begin{tabular}[c]{@{}c@{}}Indoor\\ (flying2)\end{tabular} &
\begin{tabular}[c]{@{}c@{}}Indoor\\ (flying3)\end{tabular} &
\begin{tabular}[c]{@{}c@{}}Outdoor\\ (day1)\end{tabular} &
Avg. \\ \hline
\textbf{(a)} & \checkmark & \checkmark & \ding{55} &
\color{red}\textbf{0.72$\pm$0.031} & \color{red}\textbf{1.38$\pm$0.026} &
\color{red}\textbf{1.09$\pm$0.03} & \color{red}\textbf{0.47$\pm$0.053} &
\color{red}\textbf{0.91$\pm$0.012} \\
\shortstack[l]{(b)} & \checkmark & \ding{55} & \checkmark &
\color{blue}0.81$\pm$0.047 & 1.53$\pm$0.044 & 1.28$\pm$0.064 &
\color{blue}0.55$\pm$0.041 & 1.04$\pm$0.019 \\
\shortstack[l]{(c)} & \ding{55} & \checkmark & \ding{55} &
\color{blue}0.81$\pm$0.075 & \color{blue}1.51$\pm$0.029 &
\color{blue}1.24$\pm$0.010 & 0.56$\pm$0.061 & \color{blue}1.03$\pm$0.043 \\
\shortstack[l]{(d)} & \checkmark & \ding{55} & \ding{55} &
\color{blue}0.81$\pm$0.021 & 1.52$\pm$0.059 & 1.28$\pm$0.025 &
0.61$\pm$0.055 & 1.06$\pm$0.036 \\
\bottomrule
\end{tabular}}
\end{table}

\section{Conclusion}


In this paper, we propose SICAF to address the challenges faced by SNNs in self-supervised optical flow estimation. SICAF introduces a novel ``Compactly Compress, Time-varying Focus'' training strategy and designs an estimator equipped with learnable temporal attention. Through a temporally modulated constrained optimization process, SICAF enables SNNs to extract compact latent representations while preserving the temporal heterogeneity of input events, thereby enhancing their ability to capture irregular spatiotemporal structures in event streams. 
Overall, SICAF provides a new pathway toward more efficient feature extraction and enhanced robustness in SNNs.

\section*{Acknowledgment}

This work was supported by the National Natural Science Foundation of China under Grant Nos. 62622610, 62376185 and U25A20540.

{
\bibliographystyle{IEEEtran}
\bibliography{reference}
}







\appendix

\section{Technical appendices and supplementary material}
\subsection{Broader Impacts}

This paper focuses on improving the compression process of traditional IB to achieve both compactness and diversity of features, thereby enhancing the generalization and robustness of SNN models in optical flow estimation tasks. In general, this work has no negative societal impacts.

\subsection{Limitation}

The proposed SICAF framework assists SNNs in achieving both compactness and diversity in feature extraction for optical flow estimation tasks. It addresses the issue of information loss that arises from traditional IB methods. However, in practical applications, SICAF may require parameter tuning to achieve better accuracy, which could be time-consuming. 

\subsection{Experiment Settings}
\label{settingsup}

All experiments were conducted using the PyTorch framework on RTX 4090 and A6000 GPUs. Adam optimizer is used with learning rate of 0.0002 and batch size of 8 for 50 epochs. Gradients are clipped based on a global norm of 100, and the number of events for each input partition is set to 1000. 10,000 events are used for each training event partition. The scaling weight for \(L_{smooth}\) is set as 0.001. Nearly 15 minutes of event data are contained in the training sequences, and they are split into 140 sequences with 128×128 size with 500k events in each sequence. Random horizontal, vertical, and polarity flips are used to augment data. The SNN employs a direct training approach using surrogate gradients. MVSEC dataset \cite{zhu2018multivehicle} is first used for evaluation with the ground-truth data. Optical flow predictions are produced at each frame timestep with \textit{dt}=2, i.e., 25\% of the window events at a time. To compare with ground-truth data, we convert predicted optical flow from units of pixels/partitions to pixel displacement by a multiplication factor \(dt_{gt}\)/\(dt_{input}\). ECD and HQF datasets are also used for evaluation of SICAF. Since there is no ground-truth data for these two datasets, the quality of the estimated optical flow is evaluated according to metrics based on the contrast maximization framework \cite{gallego2018unifying, gallego2019focus}. The average test accuracy over two runs under different random seeds and is reported as final test accuracy.

For robustness evaluation, we consider both black-box and white-box adversarial attacks. For the black-box attack, following ``Adversarially Robust Spiking Neural Networks Through Conversion'', we adopt a random-direction zeroth-order gradient estimation strategy. Candidate adversarial inputs are constructed using independent Gaussian random directions $\delta \sim \mathcal{N}(0,I)$ as $x' = x + \epsilon\delta$. The number of queries is set to $5$, the perturbation magnitude is set to $\epsilon = 0.1$, and the perturbed input is finally clipped to the range $[0,1]$. For the white-box attack, following ``A Comprehensive Analysis on Adversarial Robustness of Spiking Neural Networks'', we employ the FGSM. The perturbation is generated according to the sign of the gradient of the loss with respect to the input, with the perturbation magnitude set to $\epsilon = 0.1$.

\subsection{Additional Details on SNN Neurons}
\label{sec:rationale}
In LIF neurons, when the membrane potential $V$ reaches the threshold $m$, the neuron generates an output spike. The membrane potential threshold can be either a fixed value or can vary with the input. To better handle the complex motion changes in optical flow estimation, we set the membrane potential threshold $m$ and the leakage parameter $\lambda$ of the LIF neurons as learnable parameters to achieve improved performance. In addition, we adjust the neuron model to reduce adaptation delay and adapt more quickly to the rapidly changing motion characteristics in optical flow estimation. The output current $I$ and membrane potential threshold $m$ of the adjusted neuron are as follows:
\begin{align}
I_{i}^{k}&=\sum_{i} W_{ij}^{f} \delta_{j}^{k}+\sum_{l} W_{il}^{r} u_{j} \delta_{l}^{k - 1}-\kappa P_{i}^{k}, \label{eq:1} 
\end{align}
\begin{align}
P_{i}^{k}&=\eta P_{i}^{k - 1}+\frac{1-\eta}{\left|R_{i}\right|} \sum_{j \in R_{i}} \delta_{j}^{k - 1}, \label{eq:2} 
\end{align}
\begin{align}
m_{i}^{k}&=\alpha_{0}+\alpha_{1} P_{i}^{k}, \label{eq:3}
\end{align}
where $R_{i}$ is the set of receptive fields of neuron $i$ across all channels. For convenience, we list the hyperparameters of the neurons and provide their corresponding initialization value ranges, as shown in Table~\ref{tab:snn_hyperparam}.

\begin{table}[h!]
\centering
\caption{SNN hyperparameter initialization table}
\label{tab:snn_hyperparam}
\begin{tabular}{cc}
\hline
\textbf{Hyperparameter} & \textbf{Initialization range}  \\ \hline
$\alpha_0$ (learnable) & 0.01 \\ 
$\alpha_1$ (learnable) & 1 \\ 
$\nu$ (learnable) & $\nu = \frac{1}{1+\exp(-a)}, a \in \mathcal{N}(-4, 0.1)$ \\ 
$\delta$ & $\delta \in \{0, 1\}$ \\ 
$\eta$ (learnable) & $\eta = \frac{1}{1+\exp(-c)}, c \in \mathcal{N}(-4, 0.1)$ \\
$m$ & 0.8 \\
$\kappa$ (learnable) & 0.01 \\
$P$ & 0 \\
\hline
\end{tabular}
\end{table}

\subsection{Pseudocode}
\label{pseudocode}

The overall network in this paper processes event streams through four strided convolutional encoders, two residual blocks, and four decoder layers. SICAF trains the network parameters by maximizing the conditional entropy of intermediate variables, balancing the feature compactness and diversity of latent space.

\vspace{-5pt}
\begin{algorithm}[H]
\caption{The training process of the SICAF structure learning framework}
\label{alg:SICAF_training}

\KwIn{Event data ($X$); label ($Y$)}
\KwOut{Updated network weights ($W$); updated microscopic parameters
       ($\lambda', \gamma'$)}

\textbf{Parameters:} SNN neuron parameters ($\lambda', \gamma'$);
SICAF hyperparameters ($\lambda, \gamma$);
batch size ($B$); training epochs ($TE$);
learning rate ($r$); network depth ($H$);
neuronal membrane potential $V$ and spike output $O$\;

\textbf{Forward:}\;

Initialize the model and optimizer\;

Preprocess the input data to obtain batch data\;

\For{$i \gets 1$ \KwTo $TE$}{
    Fetch a mini-batch of size $B$\;

    \For{$h \gets 1$ \KwTo $H-1$}{
        Calculate $V_h$ and $O_h$\;
    }

    Calculate $V_H$ and $\hat{Y}$\;

    \tcp{Calculate the loss}
    $L_{\text{loss}}
    \gets
    \operatorname{SICAF}
    \bigl(
        L_{\text{flow}}(Y,\hat{Y}),
        \lambda,
        I(X;Z),
        \gamma,
        HT_{\alpha}(Z)
    \bigr)$\;
}

\textbf{Backward:}\;

\For{$h \gets H$ \KwTo $1$}{
    \eIf{$h \neq H$}{
        \tcp{Update the network weights}
        $\Delta W_h
        =
        \sum_k
        \frac{\partial L_{\text{loss}}}{\partial O_h^k}
        \frac{\partial O_h^k}{\partial U_h^k}
        \frac{\partial U_h^k}{\partial W_h^k}$\;

        $W_h \gets W_h-r\Delta W_h$\;

        \tcp{Update the microscopic parameters}
        $\Delta\lambda'_h
        =
        \sum_k
        \frac{\partial L_{\text{loss}}}{\partial O_h^k}
        \frac{\partial O_h^k}{\partial U_h^k}
        \frac{\partial U_h^k}{\partial\lambda'_h}$\;

        $\Delta\beta'_h
        =
        \sum_k
        \frac{\partial L_{\text{loss}}}{\partial O_h^k}
        \frac{\partial O_h^k}{\partial U_h^k}
        \frac{\partial U_h^k}{\partial\beta'_h}$\;

        $\lambda'_h
        \gets
        \lambda'_h-r\Delta\lambda'_h$\;

        $\beta'_h
        \gets
        \beta'_h-r\Delta\beta'_h$\;
    }{
        $\Delta W_h
        =
        \sum_k
        \frac{\partial L_{\text{loss}}}{\partial O_h^k}
        \frac{\partial O_h^k}{\partial U_h^k}
        \frac{\partial U_h^k}{\partial W_h^k}$\;

        $W_h \gets W_h-r\Delta W_h$\;
    }
}
\end{algorithm}

Meanwhile, the microscopic parameters in adaptive neurons adjust throughout the training process, further optimizing the network's weight parameters. The general training process is outlined in Algorithm~\ref{alg:SICAF_training}.

\subsection{Implementation Details of the SICAF Objective}
\label{app:sicaf_implementation}

The encoder output is denoted by $\mu_\phi(x)$, and the
variational encoding distribution is defined as
\begin{equation}
E_\phi(z\mid x)
=
\mathcal{N}\left(z;\mu_\phi(x),\sigma^2 I\right),
\qquad
\sigma^2
=
\exp(\mathrm{logvar}_t)+10^{-10},
\end{equation}
where $\mathrm{logvar}_t$ is learnable and initialized to $-1$. During training, we directly use $\mu_\phi(x)$ as the bottleneck representation to reduce estimation
variance. For a mini-batch of size $B$, we approximate the marginal distribution using kernel density estimation (KDE):
\begin{equation}
r(z)
=
\frac{1}{B}\sum_{j=1}^{B}
\mathcal{N}\left(z;\mu_j,\sigma^2 I\right),
\end{equation}
where $\mu_j=\mu_\phi(x_j)$. With $D_{ij}=\|\mu_i-\mu_j\|_2^2$, we use the following mean-based surrogate for the information-bottleneck term:
\begin{equation}
\widehat{I}(X;Z)
=
\log B
-
\frac{1}{B}\sum_{i=1}^{B}
\log\left[
\sum_{j=1}^{B}
\exp\left(-\frac{D_{ij}}{2\sigma^2}\right)
\right].
\end{equation}
This expression is obtained by evaluating $\log E_\phi(z\mid x_i)-\log r(z)$ at $z=\mu_i$, rather than computing the full expectation over $E_\phi(z\mid x_i)$. The connection to the information-bottleneck objective is given by the identity:
\begin{equation}
\mathbb{E}_{x}\left[
\mathrm{KL}\left(E_\phi(z\mid x)\,\|\,r(z)\right)
\right]
=
I(X;Z)
+
\mathrm{KL}\left(p(z)\,\|\,r(z)\right).
\end{equation}
Using KDE to approximate $p(z)$, we neglect the marginal-distribution mismatch term and use the mean-based surrogate above in training. This introduces an approximation rather than an exact KL evaluation. For entropy regularization, we use the RBF kernel:
\begin{equation}
K_{ij}
=
\exp\left(
-\frac{\|Z_i-Z_j\|_2^2}{2\sigma_{\mathrm{ker}}^2}
\right),
\end{equation}
with a fixed bandwidth $\sigma_{\mathrm{ker}}=1.0$,
which is distinct from the learnable encoding
standard deviation $\sigma$.
We set $\beta=4.8\times10^{-3}$ and
$\gamma=4\times10^{-4}$.
Moreover, $I_c=H_c=0$ are fixed constants and therefore
do not affect model gradients in the penalized
objective with fixed $\beta$ and $\gamma$.

\subsection{Theoretical Guarantees of SICAF}
\label{app:theory}

In this section, we provide a generalization bound and a corresponding
sample-complexity result for SICAF. Let
\[
S=\{X_i\}_{i=1}^{n}\sim \mathcal{D}^{n}
\]
be an i.i.d.\ training sample drawn from the data distribution
$\mathcal{D}$. The encoder $f_{\theta}$ maps the input $X$ to a latent
representation
\[
Z=f_{\theta}(X),
\]
which is subsequently processed by a predictor $h$. We denote the
population and empirical risks by
\begin{equation}
\mathcal{J}_{\mathrm{SICAF}}(h)
=
\mathbb{E}_{X\sim\mathcal{D}}
\left[
\ell\bigl(h(Z)\bigr)
\right],
\label{eq:population-risk}
\end{equation}
and
\begin{equation}
\widehat{\mathcal{J}}_{\mathrm{SICAF},S}(h)
=
\frac{1}{n}
\sum_{i=1}^{n}
\ell\bigl(h(Z_i)\bigr),
\qquad
Z_i=f_{\theta}(X_i),
\label{eq:empirical-risk}
\end{equation}
respectively.

For each input $X$, let
\[
\mathbf{a}(X)
=
\bigl(a_1(X),\ldots,a_T(X)\bigr)^{\top}
\in \mathbb{R}^{T}
\]
denote the temporal-attention vector produced by the temporal-attention
estimator. We assume that
\begin{equation}
a_t(X)\geq 0,
\qquad
\sum_{t=1}^{T}a_t(X)=1.
\label{eq:attention-simplex}
\end{equation}
The corresponding time-sensitive R\'enyi entropy is
\begin{equation}
H_{T}^{\alpha}(Z)
=
\sum_{t=1}^{T}
a_t(X)H_{\alpha}(Z_t).
\label{eq:time-sensitive-renyi}
\end{equation}

We define the mean temporal-attention vector as
\begin{equation}
\bar{\mathbf{a}}
=
\mathbb{E}_{X\sim\mathcal{D}}
\left[
\mathbf{a}(X)
\right],
\label{eq:mean-attention}
\end{equation}
and define the effective temporal-attention fluctuation by
\begin{equation}
\Phi_{\mathrm{eff}}
=
\mathbb{E}_{X\sim\mathcal{D}}
\left[
\left\|
\mathbf{a}(X)-\bar{\mathbf{a}}
\right\|_{2}^{2}
\right].
\label{eq:phi-eff}
\end{equation}
Moreover, let
\begin{equation}
\Delta_H
=
\left|
H_{T}^{\alpha}(Z)
-
\widehat{H}_{T}^{\alpha}(Z)
\right|
\label{eq:delta-h}
\end{equation}
denote the discrepancy between the population time-sensitive R\'enyi
entropy and its finite-sample estimate.

\begin{theorem}[Generalization bound of SICAF]
\label{thm:sicaf-generalization}
Let $\mathcal{F}$ be the hypothesis class induced by the encoder
$f_{\theta}$ and predictor $h$. Assume that:

\begin{enumerate}
    \item the latent representation is norm-bounded, i.e.,
    \begin{equation}
    \|Z\|_{2}\leq B
    \quad \text{almost surely};
    \label{eq:latent-bound}
    \end{equation}

    \item the loss function $\ell$ is $L$-Lipschitz with respect to the
    latent representation;

    \item the temporal-attention vector satisfies
    \eqref{eq:attention-simplex};

    \item the time-sensitive R\'enyi entropy is bounded as
    \begin{equation}
    0
    \leq
    H_{T}^{\alpha}(Z)
    \leq
    H_{\max};
    \label{eq:entropy-bound}
    \end{equation}

    \item the loss fluctuation associated with temporal modulation is
    sub-Gaussian with variance proxy $\sigma^{2}$.
\end{enumerate}

Then, for any $\delta\in(0,1)$, with probability at least $1-\delta$
over the draw of $S$, the following bound holds:
\begin{equation}
\begin{aligned}
\mathcal{J}_{\mathrm{SICAF}}(h)
-
\widehat{\mathcal{J}}_{\mathrm{SICAF},S}(h)
\leq\;&
\sqrt{
\frac{2L^{2}B^{2}}{n}
\left(
I(X;Z)
+
\ln\frac{1}{\delta}
\right)
}
\\
&+
\sqrt{
2\sigma^{2}\Phi_{\mathrm{eff}}
}
+
\gamma\Delta_H.
\end{aligned}
\label{eq:sicaf-generalization-bound}
\end{equation}
Here, $I(X;Z)$ denotes the mutual information between the input and the
latent representation, $\Phi_{\mathrm{eff}}$ quantifies the variability
of temporal attention across input samples, and $\Delta_H$ denotes the
finite-sample estimation error of the time-sensitive R\'enyi entropy.
\end{theorem}

\begin{proof}
We decompose the generalization gap into three components:
\begin{equation}
\mathcal{J}_{\mathrm{SICAF}}(h)
-
\widehat{\mathcal{J}}_{\mathrm{SICAF},S}(h)
=
\Delta_{\mathrm{IB}}
+
\Delta_{\mathrm{temp}}
+
\Delta_{\mathrm{ent}},
\label{eq:gap-decomposition}
\end{equation}
where $\Delta_{\mathrm{IB}}$ is the representation-complexity term,
$\Delta_{\mathrm{temp}}$ is the additional deviation induced by temporal
attention, and $\Delta_{\mathrm{ent}}$ is the error caused by estimating
the entropy regularizer from finitely many samples.

First, since $\ell$ is $L$-Lipschitz and $\|Z\|_{2}\leq B$, the
information-theoretic generalization inequality gives, with probability
at least $1-\delta$,
\begin{equation}
\Delta_{\mathrm{IB}}
\leq
\sqrt{
\frac{2L^{2}B^{2}}{n}
\left(
I(X;Z)
+
\ln\frac{1}{\delta}
\right)
}.
\label{eq:ib-generalization-term}
\end{equation}
This term is controlled by the information-bottleneck constraint used
in SICAF. In particular, enforcing a smaller information budget for
$I(X;Z)$ reduces the complexity contribution associated with the
latent representation.

Second, SICAF employs an input-dependent temporal-attention vector
$\mathbf{a}(X)$. Its deviation from the population mean
$\bar{\mathbf{a}}$ introduces an additional fluctuation in the
temporally weighted objective. By the assumed sub-Gaussianity of this
fluctuation and the definition in \eqref{eq:phi-eff}, a concentration
bound yields
\begin{equation}
\Delta_{\mathrm{temp}}
\leq
\sqrt{
2\sigma^{2}
\mathbb{E}_{X\sim\mathcal{D}}
\left[
\left\|
\mathbf{a}(X)-\bar{\mathbf{a}}
\right\|_{2}^{2}
\right]
}
=
\sqrt{
2\sigma^{2}\Phi_{\mathrm{eff}}
}.
\label{eq:temporal-deviation}
\end{equation}
Therefore, a more stable temporal-attention distribution leads to a
smaller value of $\Phi_{\mathrm{eff}}$ and hence a tighter bound.

Third, the SICAF objective contains the entropy regularization term
$-\gamma H_{T}^{\alpha}(Z)$. In practice, this population quantity is
replaced by the finite-sample estimate
$\widehat{H}_{T}^{\alpha}(Z)$. Therefore, the absolute deviation of the
entropy contribution is bounded by
\begin{equation}
\begin{aligned}
\Delta_{\mathrm{ent}}
&=
\gamma
\left|
H_{T}^{\alpha}(Z)
-
\widehat{H}_{T}^{\alpha}(Z)
\right|
\\
&=
\gamma\Delta_H.
\end{aligned}
\label{eq:entropy-deviation}
\end{equation}
Under \eqref{eq:entropy-bound}, this term is finite; in particular,
$\Delta_H\leq H_{\max}$ whenever both entropy quantities lie in
$[0,H_{\max}]$.

Combining \eqref{eq:ib-generalization-term},
\eqref{eq:temporal-deviation}, and \eqref{eq:entropy-deviation} in
\eqref{eq:gap-decomposition} gives
\eqref{eq:sicaf-generalization-bound}.
\end{proof}

\begin{theorem}[Sample complexity of SICAF]
\label{thm:sicaf-sample-complexity}
Suppose that the assumptions of
Theorem~\ref{thm:sicaf-generalization} hold and that
\begin{equation}
I(X;Z)\leq I_{0}.
\label{eq:information-budget}
\end{equation}
Assume further that the total temporal-modulation contribution is
bounded by
\begin{equation}
\sqrt{
2\sigma^{2}\Phi_{\mathrm{eff}}
}
+
\gamma\Delta_H
\leq
\varepsilon_{\mathrm{temp}}.
\label{eq:temporal-error-budget}
\end{equation}
For any target tolerance $\varepsilon$ satisfying
\begin{equation}
\varepsilon>\varepsilon_{\mathrm{temp}},
\label{eq:epsilon-condition}
\end{equation}
it is sufficient to choose
\begin{equation}
n
\geq
\frac{
2L^{2}B^{2}
}{
\left(
\varepsilon-\varepsilon_{\mathrm{temp}}
\right)^{2}
}
\left(
I_{0}
+
\ln\frac{1}{\delta}
\right)
\label{eq:sicaf-sample-complexity}
\end{equation}
to guarantee, with probability at least $1-\delta$, that
\begin{equation}
\mathcal{J}_{\mathrm{SICAF}}(h)
-
\widehat{\mathcal{J}}_{\mathrm{SICAF},S}(h)
\leq
\varepsilon.
\label{eq:target-generalization}
\end{equation}
\end{theorem}

\begin{proof}
By Theorem~\ref{thm:sicaf-generalization},
\begin{equation}
\begin{aligned}
\mathcal{J}_{\mathrm{SICAF}}(h)
-
\widehat{\mathcal{J}}_{\mathrm{SICAF},S}(h)
\leq\;&
\sqrt{
\frac{2L^{2}B^{2}}{n}
\left(
I(X;Z)
+
\ln\frac{1}{\delta}
\right)
}
\\
&+
\sqrt{
2\sigma^{2}\Phi_{\mathrm{eff}}
}
+
\gamma\Delta_H.
\end{aligned}
\label{eq:sample-proof-start}
\end{equation}
Using \eqref{eq:information-budget} and
\eqref{eq:temporal-error-budget}, we obtain
\begin{equation}
\mathcal{J}_{\mathrm{SICAF}}(h)
-
\widehat{\mathcal{J}}_{\mathrm{SICAF},S}(h)
\leq
\sqrt{
\frac{2L^{2}B^{2}}{n}
\left(
I_{0}
+
\ln\frac{1}{\delta}
\right)
}
+
\varepsilon_{\mathrm{temp}}.
\label{eq:sample-proof-reduced}
\end{equation}
To ensure that the right-hand side of
\eqref{eq:sample-proof-reduced} is at most $\varepsilon$, it suffices
to require
\begin{equation}
\sqrt{
\frac{2L^{2}B^{2}}{n}
\left(
I_{0}
+
\ln\frac{1}{\delta}
\right)
}
\leq
\varepsilon-\varepsilon_{\mathrm{temp}}.
\label{eq:sample-proof-condition}
\end{equation}
The right-hand side is positive by
\eqref{eq:epsilon-condition}. Squaring both sides gives
\begin{equation}
\frac{2L^{2}B^{2}}{n}
\left(
I_{0}
+
\ln\frac{1}{\delta}
\right)
\leq
\left(
\varepsilon-\varepsilon_{\mathrm{temp}}
\right)^{2}.
\label{eq:sample-proof-square}
\end{equation}
Rearranging \eqref{eq:sample-proof-square} with respect to $n$ yields
\eqref{eq:sicaf-sample-complexity}. Therefore,
\eqref{eq:target-generalization} holds with probability at least
$1-\delta$.
\end{proof}

\paragraph{Remark.}
Theorem~\ref{thm:sicaf-sample-complexity} does not discard the temporal
terms. Instead, it explicitly requires the temporal-attention
fluctuation and entropy-estimation error to satisfy the error-budget
condition in \eqref{eq:temporal-error-budget}. The original
mutual-information-only sample-complexity expression is recovered as a
special case when $\varepsilon_{\mathrm{temp}}=0$, or approximately
when the temporal contribution is negligible relative to the target
generalization tolerance.

\subsection{Additional Self-Supervised Optical Flow Estimation Results}
\label{OFER}

In addition to validating the performance of SICAF on the MVSEC dataset, we also evaluated SICAF on the ECD and HQF datasets. For datasets lacking ground-truth labels, we adopted the RSAT metric to measure the quality of optical flow estimation by the model on these datasets. Its mathematical definition is as follows:

\begin{equation}
\text{RSAT} \coloneqq \frac{\mathcal{L}_{\text{contrast}}(t_{\text{ref}}^{\text{fw}} | \mathbf{u})}{\mathcal{L}_{\text{contrast}}(t_{\text{ref}}^{\text{fw}} | \mathbf{0})}.
\end{equation}

The contrast loss $\mathcal{L}_{\text{contrast}}(t_{\text{ref}})$ is defined as:

\begin{equation}
    \mathcal{L}_{\text{contrast}}(t_{\text{ref}}) = \frac{\sum_x \bar{t}_{p'=+}(\mathbf{x}; \mathbf{u} | t_{\text{ref}})^2 + \bar{t}_{p'=-}(\mathbf{x}; \mathbf{u} | t_{\text{ref}})^2}{\sum_x \mathbb{I}[n(\mathbf{x}') > 0] + \epsilon},
\end{equation}

where $\bar{t}_{p'=+}$ and $\bar{t}_{p'=-}$ represent the positive and negative timestamps respectively, $n(\mathbf{x}')$ indicates the number of events at position $\mathbf{x}'$, and $\epsilon$ is a small constant added for numerical stability. When RSAT $< 1$, it indicates that the predicted optical flow outperforms the baseline scenario composed of null vectors.

Table~\ref{table2} and Figure~\ref{Qualitative evaluation} present the results for the ECD and HQF datasets, respectively. From the experimental results, it can be observed that SICAF also effectively enhances the model's generalization ability in self-supervised optical flow estimation on the ECD and HQF datasets. These results are attributed to the maximum entropy information bottleneck increasing the randomness of the intermediate variable $T$, allowing the model to extract useful information while giving deterministic prediction results a greater variance, thus improving the model's generalization.

\begin{table}[htbp]
\centering
\caption{Experimental results on ECD and HQF datasets. (Higher FWL and lower RSAT indicate better performance. Best in red, runner-up in blue.)}
\label{table2}
\footnotesize
\begin{tabular}{lcccc}
\toprule
\multirow{2.5}{*}{Methods}
 & \multicolumn{2}{c}{ECD} & \multicolumn{2}{c}{HQF} \\
\cmidrule(r){2-3} \cmidrule(r){4-5}
 & FWL $\uparrow$ & RSAT $\downarrow$ & FWL $\uparrow$ & RSAT $\downarrow$ \\
\midrule
LIF-EV-FlowNet \cite{hagenaars2021self}    & 1.21 & \color{blue} 0.95 & 1.24 & 0.94 \\
ALIF-EV-FlowNet \cite{hagenaars2021self}   & 1.17 & 0.98 & 1.21 & 0.98 \\
PLIF-EV-FlowNet \cite{hagenaars2021self}   & \color{blue} 1.24 & \color{blue} 0.95 & \color{blue} 1.28 & \color{blue} 0.93 \\
XLIF-EV-FlowNet \cite{hagenaars2021self}   & 1.23 & \color{blue} 0.95 & 1.25 & \color{blue} 0.93 \\
\rowcolor[HTML]{EAF5FB}
\textbf{SICAF} & \color{red} 1.26 & \color{red} 0.94 & \color{red} 1.30 & \color{red} 0.92 \\
\bottomrule
\end{tabular}
\end{table}

\begin{figure}[htbp]
    \centering
    \includegraphics[width=0.6\linewidth]{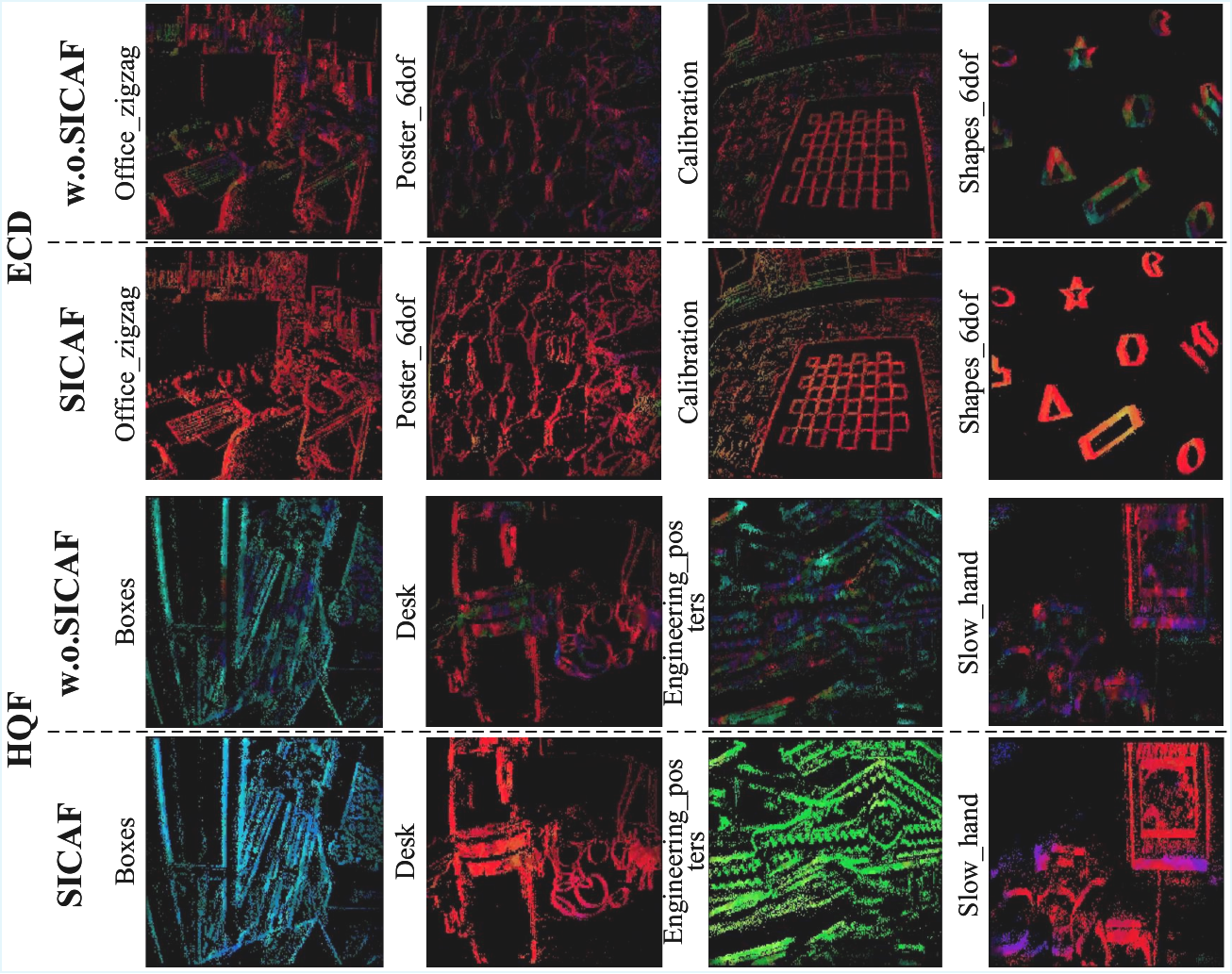}
    \caption{Qualitative evaluation of optical flow estimation results on sequences from the ECD and HQF datasets.}
    \label{Qualitative evaluation}
\end{figure}

\subsection{Additional Details on Robustness Experiments}
\label{robustness_supp}

To assess the robustness of the model, we introduced random noise attacks and adversarial noise attacks during the model testing process. Random noise attacks involve using clean datasets without additional noise during model training and adding noise to the datasets during testing to simulate disturbances in real-world scenarios. In this study, random noise attacks mainly consisted of adding two common types of noise, Gaussian noise and salt-and-pepper (SP) noise. Gaussian noise with a mean of \(\mu=0\)  and variance of \(\sigma\) can be expressed mathematically as follows:
\begin{equation}
    f_{Gaussian}(x)=\frac{1}{\sqrt{2\pi }\sigma  } e^{-\frac{(x-\mu )^2}{2\sigma ^2} },
\end{equation}
where \textit{x} is a random variable, \(\mu\) and \(\sigma\) represent its mean and variance, respectively. SP noise is composed of white and black pixels generated by severe disturbances. In addition, we explore adversarial noise attacks in two forms, i.e. white-box and black-box attacks. In white-box attack, the attacker possesses detailed information about the model, including data preprocessing, structure, and parameters, using this internal model information to generate adversarial samples. In our experiments, white-box attacks employed the fast gradient sign method (FGSM), which primarily utilizes the gradient of the loss function with respect to the input data \textit{x} to generate adversarial noise. The expression for the generated noise is as follows:
\begin{equation}
    \eta _{white-box}=\varepsilon \cdot sign(\nabla _xJ(\theta ,x,y)),
\end{equation}
where the parameter \(\varepsilon\) represents the intensity of the added noise, and \(J(\theta,x,y)\)is the loss function. \(\nabla _xJ\) denotes the gradient of the loss function with respect to the input \textit{x}. Using the FGSM algorithm, we obtain white-box adversarial noise and add it to the input dataset during the testing phase. Black-box attacks refer to scenarios where the attacker is unaware of the internal details and can only access the input and output information. Here, we use zero-order optimization methods to estimate gradients using experimental function values without directly computing or using gradient information. Zero-order optimization methods are a way to generate adversarial examples without explicit gradient information, suitable for black-box adversarial attacks. These methods optimize the input by querying the target function's output to make the model produce incorrect predictions, thus generating black-box adversarial samples. Suppose our model is $f$, and the input sample is $x$. The goal of generating adversarial attack samples is to maximize the classification error or loss of $f(x)$.

First, we initialize the input sample $\textit{x}_{\textit{adv}} = \textit{inputs}$ and generate a random direction vector $\boldsymbol{\delta} \sim \mathcal{N}(\mathbf{0}, \mathbf{I})$, where $\mathcal{N}(\mathbf{0}, \mathbf{I})$ is drawn from a multivariate normal distribution with a mean of $\mathbf{0}$ and a covariance matrix of the identity matrix. The purpose of generating the vector $\boldsymbol{\delta}$ is to explore the solution space around the input sample through random search, thereby indirectly estimating the gradient direction and finding perturbation directions that can effectively generate adversarial samples. 

We then compute the loss change for the sample when it is moved in the direction of $\boldsymbol{\delta}$. Let $L(\textit{inputs}, \textit{target})$ denote the model's loss function, where \textit{inputs} is the input sample and \textit{target} is the true label of the sample. The change in loss when the sample is moved in the direction of $\boldsymbol{\delta}$ can be expressed as:
\begin{equation}
\Delta L = L(f(\textit{x}_{\textit{adv}} + \alpha \boldsymbol{\delta}), \textit{target}) - L(f(\textit{x}_{\textit{adv}}), \textit{target}),
\end{equation}
where $\alpha$ denotes the perturbation strength. After obtaining the loss change in the direction of $\boldsymbol{\delta}$, we estimate the gradient and update the input sample as follows:
\begin{equation}
\textit{x}_{\textit{adv}} = \textit{x}_{\textit{adv}} + \alpha \boldsymbol{\delta} \cdot \textit{sign}(\Delta L),
\end{equation}

\begin{table*}[h!]
\centering
\caption{Experimental results of SICAF on ECD under different noise (lower is better, best in red).}
\label{tab:SICAF_ecd_noise}
\small
\begin{tabular}{lcccccccc}
\toprule
\textbf{Noise type} & \multicolumn{2}{c}{\textbf{SP}} & \multicolumn{2}{c}{\textbf{Gaussian}} & \multicolumn{2}{c}{\textbf{Black-box attack}} & \multicolumn{2}{c}{\textbf{White-box attack}} \\
\cmidrule(lr){2-3} \cmidrule(lr){4-5} \cmidrule(lr){6-7} \cmidrule(lr){8-9}
\textbf{Methods} & \textbf{\textbf{\citep{Hagenaars2021}}} & \textbf{SICAF} & \textbf{\textbf{ \citep{Hagenaars2021}}} & \textbf{SICAF} & \textbf{\textbf{ \citep{Hagenaars2021}}} & \textbf{SICAF} & \textbf{\textbf{ \citep{Hagenaars2021}}} & \textbf{SICAF} \\
\midrule
boxes           & 0.97  & 0.94  & 0.97  & 0.94  & 1.01  & 1.00  & 0.99  & 0.96  \\
0calibration    & 0.97  & 0.96  & 0.98  & 0.96  & 1.00  & 1.00  & 0.99  & 0.98  \\
dynamic\_6dof   & 0.94  & 0.92  & 0.95  & 0.92  & 1.00  & 0.98  & 0.97  & 0.94  \\
office\_zigzag  & 0.97  & 0.96  & 0.98  & 0.97  & 1.00  & 1.00  & 0.99  & 0.98  \\
poster\_6dof    & 0.97  & 0.94  & 0.97  & 0.94  & 1.01  & 0.99  & 0.99  & 0.90  \\
shapes\_6dof    & 1.00  & 0.99  & 1.01  & 0.99  & 1.00  & 0.99  & 1.00  & 0.98  \\
slider\_depth   & 0.91  & 0.93  & 0.91  & 0.93  & 0.99  & 0.95  & 0.92  & 0.91  \\
\midrule
\rowcolor[HTML]{EAF5FB}
\textbf{Mean}   & 0.96  & \color{red} 0.95 & 0.97  & \color{red} 0.95 & 1.01  & \color{red} 0.99 & 0.98  & \color{red} 0.96 \\
\bottomrule
\end{tabular}
\end{table*}

\begin{table*}[h!]
\centering
\caption{Experimental results of SICAF on HQF under different noise (lower is better, best in red).}
\small
\label{tab:SICAF_hqf_noise}
\begin{tabular}{lcccccccc}
\toprule
\textbf{Noise type} & \multicolumn{2}{c}{\textbf{SP}} & \multicolumn{2}{c}{\textbf{Gaussian}} & \multicolumn{2}{c}{\textbf{Black-box attack}} & \multicolumn{2}{c}{\textbf{White-box attack}} \\
\cmidrule(lr){2-3} \cmidrule(lr){4-5} \cmidrule(lr){6-7} \cmidrule(lr){8-9}
\textbf{Methods} & \textbf{\textbf{ \citep{Hagenaars2021}}} & \textbf{SICAF} & \textbf{\textbf{ \citep{Hagenaars2021}}} & \textbf{SICAF} & \textbf{\textbf{ \citep{Hagenaars2021}}} & \textbf{SICAF} & \textbf{\textbf{ \citep{Hagenaars2021}}} & \textbf{SICAF} \\
\midrule
bike\_bar\_hdr          & 0.97  & 0.95  & 0.98  & 0.95  & 1.00  & 0.98  & 0.98  & 0.95  \\
boxes                  & 0.94  & 0.92  & 0.95  & 0.93  & 1.00  & 0.97  & 0.97  & 0.93  \\
desk\_6k               & 0.92  & 0.90  & 0.93  & 0.91  & 0.99  & 0.95  & 0.95  & 0.92  \\
desk\_fast             & 0.92  & 0.89  & 0.93  & 0.90  & 0.99  & 0.95  & 0.95  & 0.91  \\
desk\_hand\_only       & 1.00  & 0.98  & 1.00  & 0.98  & 0.99  & 0.99  & 0.98  & 0.97  \\
desk\_slow             & 0.89  & 0.88  & 0.91  & 0.91  & 0.98  & 0.95  & 0.90  & 0.89  \\
engineering\_posters   & 0.96  & 0.94  & 0.97  & 0.95  & 0.99  & 0.95  & 0.99  & 0.95  \\
high\_texture\_plants  & 0.94  & 0.92  & 0.98  & 0.96  & 1.00  & 0.98  & 0.96  & 0.92  \\
poster\_pillar\_1      & 0.98  & 0.96  & 0.98  & 0.96  & 0.99  & 0.97  & 0.98  & 0.95  \\
poster\_pillar\_2      & 0.97  & 0.95  & 0.98  & 0.95  & 0.99  & 0.97  & 0.98  & 0.95  \\
reflective\_materials  & 0.97  & 0.94  & 0.97  & 0.95  & 0.99  & 0.95  & 0.98  & 0.95  \\
slow\_and\_fast\_desk  & 0.96  & 0.94  & 0.97  & 0.95  & 0.98  & 0.95  & 0.98  & 0.95  \\
slow\_hand             & 0.97  & 0.95  & 0.97  & 0.95  & 0.99  & 0.95  & 0.99  & 0.94  \\
still\_life            & 0.94  & 0.93  & 0.95  & 0.93  & 0.99  & 0.98  & 0.93  & 1.00  \\
\midrule
\rowcolor[HTML]{EAF5FB}
\textbf{Mean}          & 0.95  & \color{red} 0.93 & 0.96  & \color{red} 0.94 & 0.99  & \color{red} 0.97 & 0.96  & \color{red} 0.93 \\
\bottomrule
\end{tabular}
\end{table*}

To ensure that the generated adversarial samples are perceptually similar to the original samples, we need to keep the perturbation within a certain range to maintain the stealthiness and effectiveness of the attack. Therefore, we apply pruning operations to the generated samples:
\begin{equation}
\textit{x}_{\textit{adv}} = \textit{clip}(\textit{x}_{\textit{adv}}, \textit{inputs} - \epsilon, \textit{inputs} + \epsilon),
\end{equation}
where $\epsilon$ denotes the perturbation range set for the pruning operation. Table ~\ref{tab:SICAF_ecd_noise} and Table ~\ref{tab:SICAF_hqf_noise} show the robustness test results of SICAF on the ECD and HQF datasets.

\subsection{Computational Efficiency Analysis}
\label{app:efficiency}

We evaluate the computational efficiency of SICAF
against LF-EV-FlowNet in terms of inference latency,
inference energy per batch, memory usage, and the
number of spikes per inference batch.
Table~\ref{tab:efficiency} summarizes the results.

\begin{table}[htbp]
    \centering
    \caption{Computational efficiency comparison between
    LF-EV-FlowNet and SICAF. Lower values are preferable
    for all reported metrics.}
    \label{tab:efficiency}
    \small
    \setlength{\tabcolsep}{4pt}
    \renewcommand{\arraystretch}{1.15}
    \begin{tabular}{lcccc}
        \toprule
        Model
        & \shortstack{Latency\\(ms/batch)}
        & \shortstack{Inference energy\\per batch (J)}
        & \shortstack{Memory\\(MiB)}
        & \shortstack{Spikes per\\inference batch} \\
        \midrule
        LF-EV-FlowNet
        & $29.764 \pm 7.881$
        & $4.354 \pm 1.159$
        & $241.548 \pm 6.885$
        & $197{,}724.67 \pm 769.86$ \\
        SICAF
        & $27.164 \pm 6.645$
        & $3.873 \pm 0.436$
        & $247.503 \pm 6.582$
        & $83{,}978.00 \pm 274.99$ \\
        \bottomrule
    \end{tabular}
\end{table}

As shown in Table~\ref{tab:efficiency}, SICAF reduces
the mean inference latency from $29.764$ to $27.164$
ms/batch and the mean inference energy from $4.354$
to $3.873$ J/batch, corresponding to reductions of
approximately $8.7\%$ and $11.0\%$, respectively.
The mean number of spikes per inference batch
decreases by approximately $57.5\%$, indicating
sparser spiking activity.
Meanwhile, memory usage increases from $241.548$
to $247.503$ MiB, an increase of approximately $2.5\%$.
These results indicate lower average inference
latency and energy consumption with a modest
increase in memory usage.


\newpage
\section*{NeurIPS Paper Checklist}

\begin{enumerate}

\item {\bf Claims}
    \item[] Question: Do the main claims made in the abstract and introduction accurately reflect the paper's contributions and scope?
    \item[] Answer: \answerYes{} 
    \item[] Justification: We have described the proposed method (SICAF) and the scope of experiments in the abstract and introduction.

\item {\bf Limitations}
    \item[] Question: Does the paper discuss the limitations of the work performed by the authors?
    \item[] Answer: \answerYes{} 
    \item[] Justification: We have outlined the potential limitations in the Appendix A.2.

\item {\bf Theory assumptions and proofs}
    \item[] Question: For each theoretical result, does the paper provide the full set of assumptions and a complete (and correct) proof?
    \item[] Answer: \answerYes{} 
    \item[] Justification: The theorem and its proof are provided in the Appendix.

    \item {\bf Experimental result reproducibility}
    \item[] Question: Does the paper fully disclose all the information needed to reproduce the main experimental results of the paper to the extent that it affects the main claims and/or conclusions of the paper (regardless of whether the code and data are provided or not)?
    \item[] Answer: \answerYes{} 
    \item[] Justification: We have provided the corresponding experimental details and some parameter settings in the Appendix, along with the pseudocode of the method.

\item {\bf Open access to data and code}
    \item[] Question: Does the paper provide open access to the data and code, with sufficient instructions to faithfully reproduce the main experimental results, as described in supplemental material?
    \item[] Answer: \answerYes{} 
    \item[] Justification: Our experiments use publicly available benchmark datasets, and the code will be released to ensure that others can reproduce the results.

\item {\bf Experimental setting/details}
    \item[] Question: Does the paper specify all the training and test details (e.g., data splits, hyperparameters, how they were chosen, type of optimizer) necessary to understand the results?
    \item[] Answer: \answerYes{} 
    \item[] Justification: We have provided a detailed description of our experimental setup in the Appendix.

\item {\bf Experiment statistical significance}
    \item[] Question: Does the paper report error bars suitably and correctly defined or other appropriate information about the statistical significance of the experiments?
    \item[] Answer: \answerNo{} 
    \item[] Justification: It is worth noting that in the experimental section, we compare our results with the best-reported results from other works.

\item {\bf Experiments compute resources}
    \item[] Question: For each experiment, does the paper provide sufficient information on the computer resources (type of compute workers, memory, time of execution) needed to reproduce the experiments?
    \item[] Answer: \answerYes{} 
    \item[] Justification: We have specified the GPUs used in the experiments in the Appendix: GeForce RTX 4090 and A6000.
    
\item {\bf Code of ethics}
    \item[] Question: Does the research conducted in the paper conform, in every respect, with the NeurIPS Code of Ethics \url{https://neurips.cc/public/EthicsGuidelines}?
    \item[] Answer: \answerYes{} 
    \item[] Justification: This work follows the NeurIPS Code of Ethics.

\item {\bf Broader impacts}
    \item[] Question: Does the paper discuss both potential positive societal impacts and negative societal impacts of the work performed?
    \item[] Answer: \answerNA{} 
    \item[] Justification: This paper focuses on improving the compression process of traditional IB to achieve both compactness and diversity of features, thereby enhancing the generalization and robustness of SNN models in optical flow estimation tasks. In general, this work has no negative societal impacts.
    
\item {\bf Safeguards}
    \item[] Question: Does the paper describe safeguards that have been put in place for responsible release of data or models that have a high risk for misuse (e.g., pre-trained language models, image generators, or scraped datasets)?
    \item[] Answer: \answerNA{} 
    \item[] Justification: The paper proposes a framework for optimizing the SNN model in optical flow estimation tasks. Neither the model nor the standard benchmark data used poses such risk.

\item {\bf Licenses for existing assets}
    \item[] Question: Are the creators or original owners of assets (e.g., code, data, models), used in the paper, properly credited and are the license and terms of use explicitly mentioned and properly respected?
    \item[] Answer: \answerYes{} 
    \item[] Justification: We have properly cited the original papers corresponding to the code and datasets used in our work, and have ensured that their licenses and terms of use are respected.

\item {\bf New assets}
    \item[] Question: Are new assets introduced in the paper well documented and is the documentation provided alongside the assets?
    \item[] Answer: \answerNA{} 
    \item[] Justification: The paper does not release new assets.

\item {\bf Crowdsourcing and research with human subjects}
    \item[] Question: For crowdsourcing experiments and research with human subjects, does the paper include the full text of instructions given to participants and screenshots, if applicable, as well as details about compensation (if any)? 
    \item[] Answer: \answerNA{} 
    \item[] Justification: The paper does not involve crowdsourcing nor research with human subjects.

\item {\bf Institutional review board (IRB) approvals or equivalent for research with human subjects}
    \item[] Question: Does the paper describe potential risks incurred by study participants, whether such risks were disclosed to the subjects, and whether Institutional Review Board (IRB) approvals (or an equivalent approval/review based on the requirements of your country or institution) were obtained?
    \item[] Answer: \answerNA{} 
    \item[] Justification: The research does not involve human subjects, therefore IRB approval is not applicable.

\item {\bf Declaration of LLM usage}
    \item[] Question: Does the paper describe the usage of LLMs if it is an important, original, or non-standard component of the core methods in this research? Note that if the LLM is used only for writing, editing, or formatting purposes and does \emph{not} impact the core methodology, scientific rigor, or originality of the research, declaration is not required.
    \item[] Answer: \answerNA{} 
    \item[] Justification: The development of the core methods in this paper does not involve large language models.

\end{enumerate}

\end{document}